\let\oldvec\vec
\documentclass{llncs}
\let\vec\oldvec

\usepackage{amsfonts}
\usepackage{latexsym}
\usepackage{amssymb}
\usepackage{graphicx,psfrag}
\usepackage{fullpage}
\usepackage{framed}
\usepackage{verbatim}
\usepackage{color}
\usepackage{epsfig}
\usepackage{authblk}
\usepackage{hyperref}
\usepackage{url}
\usepackage{geometry}
\usepackage{mathtools}
\usepackage{float}
\usepackage{enumerate}
\newcommand\NN{\mathbb{Z}_{\geq 0}}

\newcommand\R{\mathbb{R}}

\newcommand{\ra}{\rightarrow}

\newcommand{\e}{\mathrm{e}}

\begin{document}
\title{A Scheme for Molecular Computation of Maximum Likelihood Estimators for Log-Linear Models}
\author{Manoj Gopalkrishnan}

\institute{Tata Institute of Fundamental Research, Mumbai 400 005, India\\
\email{manoj.gopalkrishnan@gmail.com},
\texttt{\url{http://www.tcs.tifr.res.in/~manoj}}}

\date{10 June 2016}
\maketitle

\begin{abstract}
We propose a novel molecular computing scheme 
for statistical inference. We focus 
on the much-studied statistical inference problem 
of computing maximum likelihood estimators 
for log-linear models. Our scheme takes log-linear
models to reaction systems, and the observed data to
initial conditions, so that the corresponding equilibrium of 
each reaction system encodes the corresponding 
maximum likelihood estimator.
The main idea is to exploit the coincidence between thermodynamic entropy
and statistical entropy. We map a Maximum Entropy characterization of 
the maximum likelihood estimator onto a Maximum Entropy characterization of 
the equilibrium concentrations for the reaction system. This allows for an efficient encoding of
the problem, and reveals that reaction networks are
superbly suited to statistical inference tasks.
Such a scheme may also provide a template to 
understanding how {\em in vivo} biochemical signaling pathways
integrate extensive information about their 
environment and history.
\end{abstract}

\section{Introduction}
The sophisticated behavior of cells emerges from the computations that are being performed by the underlying biochemical reaction networks. These biochemical pathways have been studied in a ``top-down'' manner, by looking for recurring motifs, and signs of modularity \cite{milo2002network}. There is also an opportunity to study these pathways in a ``bottom-up'' manner by proposing primitive building blocks which can be composed to create interesting and technologically valuable behavior. This ``bottom-up'' approach connects with work in the Molecular Computation community whose goal is to generate sophisticated behavior using  DNA hybridization reactions~\cite{seesawgates,qian2011scaling,napp2013message,yordanov2014computational,soloveichik2010dna,Cardelli_2011StrandAlgebra,cardelli2013two,qian2011efficient,benenson2004autonomous,shapiro2006bringing} and other Artificial Chemistry approaches~\cite{buisman2009computing,daniel2013synthetic}. 

We propose a new building block for molecular computation. We show that the mathematical structure of reaction networks is particularly well adapted to compute Maximum Likelihood Estimators for log-linear models, allowing a pithy encoding of such computations by reactions. According to \cite{fienberg2012maximum}: \begin{quotation}Log-linear models are arguably the most popular and important statistical models for the analysis of categorical data; see, for example, Bishop, Fienberg and Holland (1975)~\cite{bishop1975discrete}, Christensen (1997)~\cite{christensen1997log}. These powerful models, which include as special cases graphical models [see, e.g., Lauritzen (1996)~\cite{lauritzen1996graphical}] as well as many logit models [see, e.g., Agresti (2002)~\cite{agresti2013categorical}, Bishop, Fienberg and Holland (1975)~\cite{bishop1975discrete}], have applications in many scientific areas, ranging from social and biological sciences, to privacy and disclosure limitation problems, medicine, data mining, language processing and genetics. Their popularity has greatly increased in the last decades...\end{quotation}

In order to respond in a manner that maximizes fitness, a cell has to correctly estimate the overall state of its environment. Receptors that sit on cell walls collect a large amount of information about the cellular environment. Processing and integration of this spatially and temporally extensive and diverse information is carried out in the biochemical reaction pathways. We propose that this processing and integration may be advantageously viewed from the lens of machine learning. 

Our proposal entails that {\em schemes for statistical inference by reaction networks are of biological significance, and are deserving of as thorough and extensive a study as schemes for statistical inference by neural networks.} In particular, machine learning is not just a tool for the analysis of biochemical data, but theoretical and technological insights from machine learning could provide a deep and fundamental way, and perhaps ``the'' correct way, to think about biochemical networks. We view the scheme we present here as a promising first step in this program of applying machine learning insights to biochemical networks.\
\\\\
\textbf{The problem:} We illustrate the main ideas of our scheme with an example. Following \cite{pachter2005algebraic}, consider the \textbf{log-linear model} (also known as toric model) described by the \textbf{design matrix} $A = \tiny\left(\begin{array}{ccc}
2&1&0\\
0&1&2
\end{array}\right)$. This means that we are observing an event with three possible mutually exclusive outcomes, call them $X_1, X_2$, and $X_3$, which represent respectively the columns of $A$. The rows of $A$ represent ``hidden variables'' $\theta_1$ and $\theta_2$ respectively which parametrize the statistics of the outcomes in the following way specified by the columns of $A$: 
\begin{align*}
P[X_1\mid \theta_1, \theta_2] &\propto \theta_1^2\
\\P[X_2\mid\theta_1, \theta_2] &\propto \theta_1\theta_2\
\\P[X_3\mid\theta_1, \theta_2] &\propto \theta_2^2
\end{align*}
where the constant of proportionality normalizes the probabilities so they sum to $1$. \footnote{It is more common in statistics and statistical mechanics literature to write $\theta_1 = \e^{-E_1}$ and $\theta_2=\e^{-E_2}$ in terms of ``energies'' $E_1, E_2$ so that $P[X_2\mid E_1, E_2] \propto \e^{-E_1-E_2}$ for example.}

Suppose several independent trials are carried out, and the outcome $X_1$ is observed $x_1\in (0,1)$ fraction of the time, the outcome $X_2$ is observed $x_2\in(0,1-x_1)$ fraction of the time, and the outcome $X_3$ is observed $x_3 = 1 - x_1 - x_2$ fraction of the time. We wish to find the maximum likelihood estimator $(\hat{\theta}_1,\hat{\theta}_2)\in\mathbb{R}^2_{> 0}$ of the parameter $(\theta_1,\theta_2)$, i.e., that value of $\theta$ which maximizes the likelihood of the observed data. 

\textbf{Our contribution:} We describe a scheme that takes the design matrix $A$ to a reaction network that solves the maximum likelihood estimation problem. In Definition~\ref{def:mlenetwork}, we describe our scheme for every matrix $A$ over the integers with all column sums equal. All our results hold in this generality.

\begin{itemize}
\item In Definition~\ref{def:mlenetwork}.\ref{thm:MLD}, we show how to obtain from the matrix $A$, a reaction network that computes the maximum likelihood distribution. Specialized to our example, note that the kernel of the matrix $A$ is spanned by the vector $(1, -2, 1)^T$. We encode this by the reversible reaction 
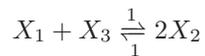
\[
X_1 + X_3 \xrightleftharpoons[1]{1} 2X_2
\] 
\item In Theorem~\ref{thm:MLD}, we show that if this reversible reaction is started at initial concentrations $X_1(0) = x_1, X_2(0) = x_2, X_3(0) = x_3$, and the dynamics proceeds according to the law of mass action with all specific rates set to $1$:
\begin{align*}
\dot{X}_1(t) = \dot{X}_3(t) = - X_1(t)X_3(t) + X_2^2(t), &&\dot{X}_2(t) = -2X_2^2(t) + 2X_1(t)X_3(t)
\end{align*}
then the reaction reaches equilibrium $(\hat{x}_1,\hat{x}_2,\hat{x}_3)$ where $\hat{x}_1 + \hat{x}_2 + \hat{x}_3 = 1$ and $\hat{x}_1 \propto \hat{\theta}_1^2$, $\hat{x}_2\propto\hat{\theta}_1\hat{\theta}_2$, and $\hat{x}_3\propto\hat{\theta}_2^2$, so that $(\hat{x}_1, \hat{x}_2, \hat{x}_3)$ represents the probability distribution over the outcomes $X_1, X_2, X_3$ at the maximum likelihood $\hat{\theta}_1,\hat{\theta}_2$.

\item This part of our scheme involves only reversible reactions, and requires no catalysis (see \cite[Theorem~5.2]{ManojCatalysis} and Lemma~\ref{lem:issaturated}). One difficulty with implementing such schemes has been that empirical control over kinetics is rather poor. Exquisitely setting the specific rates of individual reactions to desired values is very tricky, and requires a detailed understanding of molecular dynamics. Our scheme avoids this problem since any choice of specific rates that leads to the same equilibrium will do. Hence we can freely set the specific rates so long as the equilibrium constants (ratio of forward and backward specific rates) have value $1$. This is an equilibrium thermodynamic condition that is much easier to ensure in vitro. This combination of reversible reactions, no catalysis, and robustness to the values of the specific rates may make this scheme particularly easy and efficient to implement.

\item In Definition~\ref{def:mlenetwork}.\ref{def:R_MLE}, we show how to obtain from the matrix $A$ a reaction network that computes the maximum likelihood estimator. Specialized to our example, we obtain the reaction network with $5$ species $X_1, X_2, X_3, \theta_1, \theta_2$ and the $5$ reactions: 
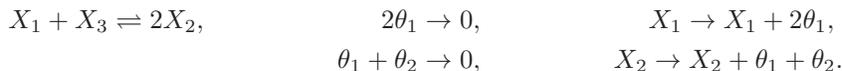
\begin{align*}
X_1 + X_3 \rightleftharpoons 2X_2, &&2\theta_1\to 0, &&X_1\to X_1 + 2\theta_1,\
\\&&\theta_1 + \theta_2\to 0, &&X_2\to X_2 + \theta_1 + \theta_2.
\end{align*}
The number of species equals the number of rows plus the number of columns of $A$. The reactions are not uniquely determined by the problem, but become so once we choose a basis for the kernel of $A$ and a maximal linearly independent set of columns. Here we have chosen columns $1$ and $2$. Each column of $A$ determines a pair of irreversible reactions.

\item Theorem~\ref{thm:MLE} implies that if this reaction system is launched at initial concentrations $X_1(0) = x_1, X_2(0) = x_2, X_3(0) = x_3$ and arbitrary concentrations of $\theta_1(0)$ and $\theta_2(0)$, and the dynamics proceeds according to the law of mass action with all specific rates set to $1$:
\begin{align*}
\dot{X}_1(t) = \dot{X}_3(t) = - X_1(t)X_3(t) + X_2^2(t), && \dot{\theta}_1(t) =-2\theta_1^2(t) + 2X_1(t) -\theta_1(t)\theta_2(t) + X_2(t),\
\\\dot{X}_2(t) = -2X_2^2(t) + 2X_1(t)X_3(t), &&\dot{\theta}_2(t) = 
-\theta_1\theta_2(t) +X_2(t),
\end{align*}
then the reaction reaches equilibrium $(\hat{x_1},\hat{x_2},\hat{x_3},\hat{\theta}_1,\hat{\theta}_2)$ where $(\hat{\theta}_1, \hat{\theta}_2)$ is the maximum likelihood estimator for the data frequency vector $(x_1, x_2, x_3)$ and $(\hat{x}_1, \hat{x}_2, \hat{x}_3)$ represents the probability distribution over the outcomes $X_1, X_2, X_3$ at the maximum likelihood. We prove global convergence: our dynamical system provably converges to the desired equilibrium. Global convergence results are known to be notoriously hard to prove in reaction network theory~\cite{GeoGac}.

\item A number of schemes have been proposed for translating reaction networks into DNA strand displacement reactions \cite{soloveichik2010dna,qian2011efficient,Cardelli_2011StrandAlgebra,cardelli2013two}. Adapting these schemes to our setting should allow molecular implementation of our MLE-solving reaction networks with DNA molecules.
\end{itemize}


\section{Maximum Likelihood Estimation in toric models}
The definitions and results in this section mostly follow \cite{pachter2005algebraic}. Because we require a slightly stronger statement, and Theorem~\ref{thm:birch} allows a short, easy, and insightful proof, we give the proof here for completeness.

In statistics, a \textbf{parametric model} consists of a family of probability distributions, one for each value of the parameters. This can be described as a map from a manifold of parameters into a manifold of probability distributions. If this map can be described by monomials as below, then the parametric statistical model is called a \textbf{toric} or \textbf{log-linear} model, as we now describe.

\begin{definition}[Toric Model]
Let $m,n$ be positive integers. The probability simplex and its relative interior are:
\[
\Delta^n:= \{(x_1,x_2,\dots,x_n)\in\mathbb{R}^n_{\geq 0} \mid x_1 + x_2 + \dots + x_n = 1\}
\]
\[
\operatorname{ri}(\Delta^n):= \{(x_1,x_2,\dots,x_n)\in\mathbb{R}^n_{> 0} \mid x_1 + x_2 + \dots + x_n = 1\}.
\] 
An $m\times n$ matrix $A = (a_{ij})_{m\times n}$ of integer entries is a \textbf{design matrix} iff all its column sums $\sum_i a_{ij}$ are equal. Let $a_j := (a_{1j}, a_{2j}, \dots, a_{mj})^T$ be the $j$'th column of $A$.  Define $\theta^{a_j}:= \theta_1^{a_{1j}}\theta_2^{a_{2j}}\dots\theta_m^{a_{mj}}$. Define the \textbf{parameter space} 
$
\Theta:=\{\theta\in\mathbb{R}^m_{>0}\mid \theta^{a_1}+\theta^{a_2}+\dots+\theta^{a_n}=1\}.
$
The \textbf{toric model} of $A$ is the map 
\[
p_A=(p_1,p_2,\dots,p_n):\Theta \to \Delta^n\text{ given by }p_j(\theta) = \theta^{a_j}\text{ for }j=1\text{ to }n.
\] 
\end{definition}

We could also have defined the parameter space $\Theta$ to be all of $\mathbb{R}^m_{>0}$, in which case we would need to normalize the probabilities by the {\em partition function} $\theta^{a_1}+\theta^{a_2}+\dots+\theta^{a_n}$ to make sure they add up to $1$. For our present purposes, the current approach will prove technically more direct.

Note that here $p_j(\theta)$ specifies $\operatorname{Pr}[j\mid \theta]$, the conditional probability of obtaining outcome $j$ given that the true state of the world is described by $\theta$.

A central problem of statistical inference is the problem of \textbf{parameter estimation}. After performing several independent identical trials, suppose the \textbf{data vector} $u\in\mathbb{Z}_{\geq 0}^n$ is obtained as a record of how many times each outcome occurred. Let the norm $|u|_1:= u_1+u_2+\dots+u_n$ denote the total number of trials performed. The \textbf{Maximum Likelihood} solution to the problem of parameter estimation finds that value of the parameter $\theta$ which maximizes the \textbf{likelihood function} $f_u(\theta):=\operatorname{Pr}[u\mid\theta]$, i.e.:
\begin{align}\label{def:mle}
\hat{\theta}(u) := \arg\sup_{\theta\in \Theta} f_u(\theta)
\end{align}
is a \textbf{maximum likelihood estimator} or MLE for the data vector $u$. We will call the point $\hat{p}(u):=p_A(\hat{\theta}(u))$ a \textbf{maximum likelihood distribution}.

\begin{definition}
Let $A$ be an $m\times n$ design matrix, and $u$ a data vector. Then the \textbf{sufficient polytope} is $P_A(u) := \{ p\in\operatorname{ri}(\Delta^n)\mid A p = A \frac{u}{|u|_1}\}$.
\end{definition}

The following theorem is a version of Birch's theorem from Algebraic Statistics. It provides a variational characterization of the maximum likelihood distribution as the unique maximum entropy distribution in the sufficient polytope. In particular the maximum likelihood distribution always belongs to the sufficient polytope, which justifies the name.

\begin{theorem}\label{thm:birch}
Fix a design matrix $A$ of size $m\times n$. 
\begin{enumerate}
\item If $u,v\in\mathbb{Z}^n_{\geq 0}$ are nonzero data vectors such that $A u/|u|_1 = A v/|v|_1$ then they have the same maximum likelihood estimator: $\hat{\theta}(u) = \hat{\theta}(v)$. 
\item Further if $P_A(u)$ is nonempty then 
\begin{enumerate}
\item There is a unique distribution $\tilde{p}\in P_A(u)$ which maximizes Shannon entropy $H(p)=-\sum_{i=1}^n p_i\log p_i$ viewed as a real-valued function from the closure $\overline{P_A(u)}$ of $P_A(u)$ with $0\log 0$ defined as $0$. 
\item $\{\tilde{p}\} = P_A(u)\cap p_A(\Theta)$. 
\item $\tilde{p}=\hat{p}(u)$, the Maximum Likelihood Distribution for the data vector $u$.  
\end{enumerate}
\end{enumerate}
\end{theorem}
\begin{proof}
1. Fix a data vector $u$. Note that $f_u(\theta) = \frac{|u|_1!}{u_1!u_2!\dots u_n!}p_1(\theta)^{u_1}p_2(\theta)^{u_2}\dots p_n(\theta)^{u_n} = \frac{|u|_1!}{u_1!u_2!\dots u_n!}\theta^{Au}$. Therefore the maximum likelihood estimator
\[
\hat{\theta}(u) = \arg\sup_{\theta\in\Theta} \theta^{Au} = \arg\sup_{\theta\in\Theta} (\theta^{Au})^{1/|u|_1} = \arg\sup_{\theta\in\Theta} \theta^{Au/|u|_1}
\]
where the second equality is true because the function $x\mapsto x^c$ is monotonically increasing whenever $c>0$. It follows that if $v\in\mathbb{Z}^n_{\geq 0}$ is a data vector such that $A u/|u|_1 = A v/|v|_1$ then $\hat{\theta }(u) = \hat{\theta}(v)$.\
\\
\\2.(a) Suppose $P_A(u)$ is nonempty. A local maximum of the restriction $H|_{\overline{P_A(u)}}$ of $H$ to the polytope $\overline{P_A(u)}$ can not be on the boundary $\partial \overline{P_A(u)}$ because for $p\in \partial \overline{P_A(u)}$, moving in the direction of arbitrary $q\in P_A(u)$ increases $H$, as can be shown by a simple calculation:
\[
\lim_{\lambda\to 0}\frac{d}{d\lambda}H((1-\lambda)p + \lambda q) \to +\infty.
\]
Since $H$ is a continuous function and the closure $\overline{P_A(u)}$ is a compact set, $H$ must attain its maximum value in $P_A(u)$. Further $H$ is a strictly concave function since its Hessian is diagonal with entries $-1/p_i$ and hence negative definite. It follows that $H|_{\overline{P_A(u)}}$ is also strictly concave, and has a unique local maximum at $\tilde{p}\in P_A(u)$, which is also the global maximum.\
\\
\\(b) By concavity of $H$, the maximum $\tilde{p}$ is the unique point in $P_A(u)$ such that $\nabla H(\tilde{p})$ is perpendicular to $P_A(u)$. We claim that $q\in P_A(u)\cap p_A(\Theta)$ iff $\nabla H(q) = (-1 - \log q_1, -1-\log q_2, \dots, -1-\log q_n)$ is perpendicular to $P_A(u)$. Since all column sums are equal, this is equivalent to requiring that $\log q$ be in the span of the rows of $A$, which is true iff $q\in p_A(\Theta)$. Hence $P_A(u)\cap p_A(\Theta)=\{\tilde{p}\}$.\
\\
\\(c) To compute the Maximum Likelihood Distribution $\hat{p}(u)$, we proceed as follows:
\begin{align*}
\hat{p}(u) &= p_A(\hat{\theta}(u)) = p_A(\arg\sup_{\theta\in\Theta} \theta^{Au}) = p_A(\arg\sup_{\theta\in\Theta} \theta^{Au/|u|_1})\ 
\\&= p_A(\arg\sup_{\theta\in\Theta} \theta^{A\tilde{p}}) =\arg\sup_{p\in p_A(\Theta)} p^{\tilde{p}}  = \arg\sup_{p\in p_A(\Theta)} \sum_{i=1}^n \tilde{p}_i\log p_i = \tilde{p}
\end{align*}
where the fourth equality uses $A\tilde{p} = Au/|u|_1$ and the last equality follows because $\sum_{i=1}^n \tilde{p}_i\log p_i$ viewed as a function of $p$ attains its maximum in all of $\Delta^n$, and hence in $p_A(\Theta)$, at $p = \tilde{p}$.
\end{proof}

This theorem already exposes the core of our idea. We will design reaction systems that maximize entropy subject to the ``correct'' constraints capturing the polytope $P_A(u)$. Then because the reactions also proceed to maximize entropy, the equilibrium point of our dynamics will correspond to the maximum likelihood distribution. Most of the technical work will go in proving convergence of trajectories to these equilibrium points.

\section{Reaction Networks}
According to \cite{Klavins_2011Biomolecular}, ``In building a design theory for chemistry, chemical reaction networks are usually the most natural intermediate representation - the middle of the
hourglass \cite{doyle2007rules}. Many different high level languages and
formalisms have been and can likely be compiled to
chemical reactions, and chemical reactions themselves (as
an abstract specification) can be implemented with a variety
of low level molecular mechanisms.''

In Subsection~\ref{subsec:crntreview}, we recall the definitions and results for reaction networks which we will need for our main results. For a comprehensive presentation of these ideas, see \cite{ManojCatalysis}. In Subsection~\ref{subsec:pert}, we prove a new result in reaction network theory. We extend a previously known global convergence result to the case of perturbations.

\subsection{Brief review of Reaction Network Theory}\label{subsec:crntreview}
For vectors $a=(a_i)_{i\in S}$ and $b=(b_i)_{i\in S}$, the notation $a^b$ will be shorthand for the formal monomial $\prod_{i\in S} a_i^{b_i}$. We introduce some standard definitions.

\begin{definition}[Reaction Network]

Fix a finite set $S$ of \textbf{species}. 
\begin{enumerate}
\item A \textbf{reaction} over $S$ is a pair $(y,y')$ such that $y,y'\in \NN^S$. It is usually written $y\ra y'$, with \textbf{reactant} $y$ and \textbf{product} $y'$. 
\item A \textbf{reaction network} consists of a finite set $S$ of species, and a finite set $\mathcal{R}$ of reactions. 
\item A reaction network is \textbf{reversible} iff for every reaction $y\to y'\in\mathcal{R}$, the reaction $y'\to y\in\mathcal{R}$.
\item A reaction network is \textbf{weakly reversible} iff for every reaction $y\to y'\in\mathcal{R}$ there exists a positive integer $n\in\mathbb{Z}_{>0}$ and $n$ reactions $y_1\to y_2,y_2\to y_3,\dots,y_{n-1}\to y_n\in\mathcal{R}$ with $y_1=y'$ and $y_n=y$. 
\item The \textbf{stoichiometric subspace} $H\subseteq\mathbb{R}^S$ is the subspace spanned by $\{y'-y\mid y\to y'\in\mathcal{R}\}$, and $H^\perp$ is the orthogonal complement of $H$. 
\item A \textbf{siphon} is a set $T\subseteq S$ of species such that for all $y\to y'\in\mathcal{R}$, if there exists $i\in T$ such that $y'_i>0$ then there exists $j\in T$ such that $y_j>0$.
\item A siphon $T\subseteq S$ is \textbf{critical} iff $v\in H^\perp\cap\mathbb{R}^S_{\geq 0}$ with $v_i=0$ for all $i\notin T$ implies $v=0$.
\end{enumerate} 
\end{definition}

\begin{definition}
Fix a weakly reversible reaction network $(S,\mathcal{R})$. The \textbf{associated ideal} $I_{(S,\mathcal{R})}\subseteq \mathbb{C}[x]$ where $x=(x_i)_{i\in S}$ is the ideal generated by the binomials $\{ x^y - x^{y'}\mid y\to y'\in\mathcal{R}\}$. A reaction network is \textbf{prime} iff its associated ideal is a prime ideal.
\end{definition}

%
The following theorem follows from \cite[Theorem~4.1, Theorem~5.2]{ManojCatalysis}.
\begin{theorem}\label{thm:prime}
A weakly reversible prime reaction network $(S,\mathcal{R})$ has no critical siphons.
\end{theorem}
We now recall the mass-action equations which are widely employed for modeling cellular processes~\cite{thomson2009unlimited,shinar2010structural,sontag01structure,tyson2003sniffers} in Biology.

\begin{definition}[Mass Action System]
A \textbf{reaction system} consists of a reaction network $(S,\mathcal{R})$ and a \textbf{rate function} $k:\mathcal{R}\to\R_{>0}$. The \textbf{mass-action equations} for a reaction system are the system of ordinary differential equations in {\em concentration} variables $\{x_i(t) \mid i\in S\}$: 
\begin{equation}\label{eqn:ma}
\dot{x}(t) = \sum_{y\to y' \in \R} k_{y\to y'}\,  x(t)^y \,(y' - y)
\end{equation}
where $x(t)$ represents the vector $(x_i(t))_{i\in S}$ of concentrations at time $t$.
\end{definition}

 Note that $\dot{x}(t)\in H$, so affine translations of $H$ are invariant under the dynamics of Equation~\ref{eqn:ma}.

We recall the well known notions of detailed balanced and complex balanced reaction system.
\begin{definition}

A reaction system $(S,\mathcal{R},k)$ is 
\begin{enumerate}
\item \textbf{Detailed balanced} iff it is reversible and there exists a point $\alpha\in\mathbb{R}^S_{>0}$ such that for every $y\to y'\in\mathcal{R}$:
\[
	 k_{y\to y'}\, \alpha^y \,(y' - y)   = k_{y'\to y}\, \alpha^{y'}\,(y - y')
\]
A point $\alpha\in\mathbb{R}^S_{>0}$ that satisfies the above condition is called a \textbf{point of detailed balance}.

\item \textbf{Complex balanced} iff there exists a point $\alpha\in\mathbb{R}^S_{>0}$ such that for every $y\in\mathbb{Z}^S_{\geq 0}$:
\[
	 \sum_{y\to y'\in \mathcal{R}} k_{y\to y'}\, \alpha^y \,(y' - y)   = \sum_{y''\to y\in\mathcal{R}} k_{y''\to y}\, \alpha^{y''}\,(y - y'')
\]
A point $\alpha\in\mathbb{R}^S_{>0}$ that satisfies the above condition is called a \textbf{point of complex balance}.
\end{enumerate}
\end{definition}

The following observations are well known and easy to verify.
\begin{itemize}
\item A complex balanced reaction system is always weakly reversible. 
\item If all rates $k_{y\to y'}=1$ and the network is weakly reversible then the reaction system is complex balanced with point of complex balance $(1,1,\dots,1)\in\mathbb{R}^S$; if the network is reversible then the reaction system is also detailed balanced with point of detailed balance $(1,1,\dots,1)\in\mathbb{R}^S$. 
\item Every detailed balance point is also a complex balance point, but there are complex balanced reversible networks that are not detailed balanced.
\end{itemize}

It is straightforward to check that every point of complex balance (respectively, detailed balance) is a fixed point for Equation~\ref{eqn:ma}. The next theorem, which follows from \cite[Theorem~2]{Angeli2007598} and \cite{horn74dynamics}, states that a converse also exists: if a reaction system is complex balanced (respectively, detailed balanced) then every fixed point is a point of complex balance (detailed balance). Further there is a unique fixed point in each affine translation of $H$, and if there are no critical siphons then the basin of attraction for this fixed point is as large as possible, namely the intersection of the affine translation of $H$ with the nonnegative orthant.

\begin{theorem}[Global Attractor Theorem for Complex Balanced Reaction Systems with no critical siphons]\label{thm:gac}
Let $(S,\mathcal{R},k)$ be a weakly reversible  complex balanced reaction system with no critical siphons and point of complex balance $\alpha$. Fix a point $u\in\mathbb{R}^S_{>0}$. Then there exists a point of complex balance $\beta$ in $(u+H)\cap\mathbb{R}^S_{>0}$ such that for every trajectory $x(t)$ with initial conditions $x(0)\in (u+H)\cap\mathbb{R}^S_{\geq 0}$, the limit $\lim_{t\to\infty} x(t)$ exists and equals $\beta$. Further the function $g(x) := \sum_{i=1}^n x_i\log x_i - x_i-x_i\log\alpha_i$ is strictly decreasing along non-stationary trajectories and attains its unique minimum value in $(u+H)\cap\mathbb{R}^S_{\geq 0}$ at $\beta$.
\end{theorem}
It is not completely trivial to show, but nevertheless true, that this theorem holds with weakly reversible replaced by ``reversible'' and ``complex balance'' replaced by ``detailed balance.'' What is to be shown is that the point of complex balance obtained in $(u+H)\cap\mathbb{R}^S_{\geq 0}$ by minimizing $g(x)$ is actually a point of detailed balance, and this follows from an examination of the form of the derivative $\frac{d}{dt}g(x(t))$ along trajectories $x(t)$ to Equation~\ref{eqn:ma}.

\subsection{A Perturbatively-Stable Global Attractor Theorem}\label{subsec:pert}
Global attractor results usually assume that the reaction network is weakly reversible. We are going to describe our scheme in the next section. 
Our scheme will employ reaction networks that are not weakly reversible, yet we will prove global attractor results for them. The key idea we use is that our reaction network can be broke into a reversible part, and an irreversible part. The reversible part acts on, but evolves independent of, the irreversible part. So we get to use the global attractor results ``as is'' on the reversible part. Further, as the reversible part approaches equilibrium, our irreversible part behaves as a perturbation of a reversible detailed-balanced network. The closer the reversible part gets to equilibrium, the smaller the perturbation of the irreversible part from the dynamics of a certain reversible detailed-balanced network. 

To make this proof idea work out, we will need a perturbative version of Theorem~\ref{thm:gac}. The next lemma shows that if the rates are perturbed slightly then, outside a small neighborhood of the detailed balance point, the strict Lyapunov function $g(x)$ from Theorem~\ref{thm:gac} continues to decrease along non-stationary trajectories.

\begin{lemma}\label{lem:striclya}
Let $(S,\mathcal{R},k)$ be a weakly reversible complex balanced reaction system with no critical siphons and point of complex balance $\alpha$. For every sufficiently small $\epsilon>0$ there exists $\delta>0$ such that for all $x'$ outside the $\epsilon$-neighborhood of $\alpha$ in $(\alpha+H)\cap\mathbb{R}^S_{\geq 0}$, the derivative $\frac{d}{dt}g(x(t))|_{t=0}<-\delta$, where $x(t)$ is a solution to the Mass-Action Equations~\ref{eqn:ma} with $x(0)=x'$.
\end{lemma}
\begin{proof}
Let $B_\epsilon$ be the open $\epsilon$ ball around $\alpha$ in $(\alpha+H)\cap\mathbb{R}^S_{\geq 0}$, with $\epsilon$ small enough so that $B_\epsilon$ does not meet the boundary $\partial\mathbb{R}^S_{\geq 0}$. Consider the closed set $S:= (\alpha+H)\cap\mathbb{R}^S_{\geq 0}\setminus B_\epsilon$. Define the orbital derivative of $g$ at $x'$ as $\mathcal{O}_k g(x'):=\frac{d}{dt}g(x(t))|_{t=0}$, where $x(t)$ is a solution to the mass-action equations~\ref{eqn:ma} with $x(0)=x'$.  Define $\delta:= \inf_{x'\in S} (- \mathcal{O}_k g(x'))$. If $\delta\leq 0$ then since $S$ is a closed set, and $\mathcal{O}_k g$ is a continuous function, there exists a point $x'$ such that $\mathcal{O}_k g(x')\geq 0$, which contradicts Theorem~\ref{thm:gac}. 
\end{proof}

We formalize the notion of perturbation using \textbf{differential inclusions}. Recall that differential inclusions model uncertainty in dynamics in a nondeterministic way by generalizing the notion of vector field. A differential inclusion maps every point to a subset of the tangent space at that point.

\begin{definition}\label{def:pert}
Let $(S,\mathcal{R},k)$ be a reaction system and let $\delta>0$. The $\delta$-\textbf{perturbation} of $(S,\mathcal{R},k)$ is the differential inclusion $V:\mathbb{R}^S_{\geq 0}\to 2^{\mathbb{R}^S}$ that at point $x\in\mathbb{R}^S_{\geq 0}$ takes the value 
\[
V(x):=\left\{ \sum_{y\to y'\in \mathcal{R}} k'_{y\to y'} x^y (y'-y) \,\,\,\middle|\,\,\, k'_{y\to y'} \in (k_{y\to y'} - \delta, k_{y\to y'}+\delta)\text{ for all }y\to y'\in\mathcal{R}\right\}.
\]
A \textbf{trajectory} of $V$ is a tuple $(I,x)$ where $I\subseteq\mathbb{R}$ is an interval and $x:I\to \mathbb{R}^S_{\geq 0}$ is a differentiable function  with $\dot{x}(t)\in V(x(t))$.
\end{definition}

\begin{theorem}[Perturbatively-Stable Global Attractor Theorem for Complex Balanced Reaction Systems with no critical siphons]\label{thm:pgac}
Let $(S,\mathcal{R},k)$ be a weakly reversible complex balanced reaction system with no critical siphons. Fix a point $u\in\mathbb{R}^S_{>0}$. Then there exists a point of complex balance $\beta$ in $(u+H)\cap\mathbb{R}^S_{>0}$ such that:
\begin{enumerate}
\item For every sufficiently small $\varepsilon>0$, there exists $\delta>0$ such that every trajectory of the form $(\mathbb{R}_{\geq 0},x)$ to the $\delta$-perturbation of $(S,\mathcal{R},k)$ with initial conditions $x(0)\in (u+H)\cap\mathbb{R}^S_{\geq 0}$ eventually enters an $\varepsilon$-neighborhood of $\beta$ and never leaves.

\item Consider a sequence $\delta_1> \delta_2> \dots >0$ and a sequence $0<t_1< t_2< \dots$ such that $\lim_{i\to\infty}\delta_i=0$ and $\lim_{i\to\infty} t_i = +\infty$, and a trajectory $(\mathbb{R}_{\geq 0}, x)$ with $x(0)\in (u+H)\cap\mathbb{R}^S_{\geq 0}$ such that $((t_i,\infty),x)$ is a trajectory of the $\delta_i$-perturbation of $(S,\mathcal{R},k)$. Then the limit $\displaystyle\lim_{t\to\infty} x(t)=\beta$.

\end{enumerate}
\end{theorem}
\begin{proof}[Proof sketch]
1. Fix $\varepsilon>0$ such that the $\varepsilon$-ball $B_\varepsilon$ around $\beta$ does not meet the boundary $\partial\mathbb{R}^S_{\geq 0}$. By Lemma~\ref{lem:striclya}, outside $B_\varepsilon$, there exists $\delta_\varepsilon>0$ such that the function $\mathcal{O}_kg<-\delta_\varepsilon$. Since $\mathcal{O}_kg$ is a continuous function of the specific rates $k$, a sufficiently small perturbation $\delta>0$ in the rates will not change the sign of $\mathcal{O}_kg$. Hence, outside $B_\epsilon$, the function $g$ is strictly decreasing along trajectories $x(t)$ to Equation~\ref{eqn:ma}. It follows that eventually every trajectory must enter $B_\epsilon$.

2. Fix a sequence  $\varepsilon_1>\varepsilon_2>\dots>0$ with $\varepsilon_1$ small enough so that the $\varepsilon_1$-ball around $\beta$ does not meet the boundary $\partial\mathbb{R}^S_{\geq 0}$ and $\lim_{i\to\infty}\varepsilon_i\to 0$. For each $\varepsilon_i$, there exists $j$ such that $\delta_j$ is small enough as per part (1) of the theorem. So every trajectory will eventually enter the $\epsilon_i$ neighborhood of $\beta$, and never leave. Since this is true for every $i$ and $\lim_{i\to \infty}\varepsilon_i\to 0$, the result follows.
\end{proof}

\section{Main Result}
The next definition makes precise our scheme, which takes a design matrix $A$ to a reaction system $\mathcal{S}_{MLE}$ depending on $A$. The choice of this reaction system is not unique, but depends on two choices of basis. We proceed in two stages. In the first stage, we construct the reaction system $\mathcal{S}_{MLD}$ which solves the problem of finding the maximum likelihood distribution. In the second stage, we add reactions to solve for $\theta$ from the algebraic relations between the $\theta$ and $X$ variables, obtaining $\mathcal{S}_{MLE}$.

\begin{definition}\label{def:mlenetwork}
Fix a design matrix  $A= (a_{ij})_{m\times n}$, a basis $B$ for the free group $\mathbb{Z}^n\cap\ker A$, and a maximal linearly-independent subset $B'$ of the columns of $A$.
\begin{enumerate}
\item\label{def:R_MLD} The reaction network $\mathcal{R}_{MLD}(A,B)$ consists of $n$ species $X_1, X_2,\dots, X_n$ and for each $b\in B$, the reversible reaction: 
\[
\sum_{j:b_j>0} b_j X_j \rightleftharpoons \sum_{j:b_j<0} -b_j X_j
\]
\item\label{def:R_MLE} The reaction system $\mathcal{S}_{MLD}(A,B)$ consists of the reaction network $\mathcal{R}_{MLD}(A,B)$ with an assignment of rate $1$ to each reaction.

\item The reaction network $\mathcal{R}_{MLE}(A,B,B')$ consists of $m+n$ species $\theta_1, \theta_2,\dots, \theta_m, X_1, X_2,\dots, X_n$, and in addition to the reactions in $\mathcal{R}_{MLD}$, the following reactions: 
\begin{itemize}
\item For each column $j\in B'$ of $A$, a reaction $\sum_{i=1}^m a_{ij} \theta_i \to 0$.
\item For each column $j\in B'$ of $A$, a reaction $X_j \to X_j + \sum_{i=1}^m a_{ij} \theta_i$.
\end{itemize}

\item The reaction system $\mathcal{S}_{MLE}(A,B,B')$ consists of the reaction network $\mathcal{R}_{MLE}(A,B,B')$ with an assignment of rate $1$ to each reaction.
\end{enumerate}
\end{definition}

Note that by the rank-nullity theorem of linear algebra, the dimension of the kernel plus the rank of the matrix equals the number of columns of the matrix. Hence counting the reversible reactions as two irreversible reactions, our scheme yields a reaction system whose number of reactions is twice the number of columns of $A$.

It is clear from the definition of $\mathcal{S}_{MLE}$ that the reactions that come from $\mathcal{R}_{MLD}$ are reversible and evolve without being affected by the other reactions. Hence we first prove global convergence of the reaction system $\mathcal{S}_{MLD}$ to the maximum likelihood distribution. This part is fairly straightforward. The key point is to verify that the reaction network $\mathcal{R}_{MLD}$ has no critical siphons. In fact, we show in the next lemma that $\mathcal{R}_{MLD}$ is prime, which will imply ``no critical siphons'' by Theorem~\ref{thm:prime}.

\begin{lemma}\label{lem:issaturated}
Fix a design matrix  $A= (a_{ij})_{m\times n}$ and a basis $B$ for the free group $\mathbb{Z}^n\cap\ker A$. Then the reaction network $\mathcal{R}_{MLD}(A,B)$ is prime and $\mathcal{S}_{MLD}(A,B)$ is  detailed balanced. Consequently, the reaction system $\mathcal{S}_{MLD}(A,B)$ is globally asymptotically stable.
\end{lemma}
\begin{proof}
$\mathcal{R}_{MLD}(A,B)$ is prime by \cite[Corollary~2.15]{miller2011theory}. The idea is to look at the toric model $p_A$ as a ring homomorphism $\mathbb{C}[x_1,x_2,\dots, x_n]\to \mathbb{C}[\mathbb{N} A]$ with $x_j\mapsto \theta^{a_j}$. (Here $\mathbb{N}A$ is the affine semigroup generated by the columns of $A$.) The kernel of this ring homomorphism is the associated ideal of $\mathcal{R}_{MLD}(A,B)$ by \cite[Proposition~2.14]{miller2011theory}, and the codomain is an integral domain, so the kernel must be prime.

To verify that $\mathcal{S}_{MLD}(A,B)$ is detailed balanced, note that the point $(1,1,\dots,1)\in\mathbb{R}^n$ is a point of detailed balance since all rates are $1$. Global asymptotic stability now follows from Theorem~\ref{thm:prime} and Theorem~\ref{thm:gac}.
\end{proof}

We can now obtain global convergence for $\mathcal{S}_{MLD}$.

\begin{theorem}[The reaction system $\mathcal{S}_{MLD}(A,B)$ computes the Maximum Likelihood Distribution]\label{thm:MLD}
Fix a design matrix  $A= (a_{ij})_{m\times n}$, a basis $B$ for the free group $\mathbb{Z}^n\cap\ker A$, and a nonzero data vector $u\in\mathbb{Z}^n_{\geq 0}$. Let $x(t) = (x_1(t), x_2(t),\dots, x_n(t))$ be a solution to the mass-action differential equations for the reaction system $\mathcal{S}_{MLD}(A,B)$ with initial conditions $x(0) = u/|u|_1$. Then $x(\infty):=\displaystyle\lim_{t\to\infty} x(t)$ exists and equals the maximum likelihood distribution $\hat{p}(u)$.
\end{theorem}
\begin{proof}
For the system $\mathcal{S}_{MLD}(A,B)$, note that $(x(0)+H)\cap\mathbb{R}^n_{>0}= P_A(u/|u|_1)$. By Theorem~\ref{thm:gac}, $x(\infty)$ exists, and the function $\sum_{i=1}^n x_i\log x_i - x_i - x_i\log 1$ attains its unique minimum in $P_A(u/|u|_1)$ at $x(\infty)$. Since the system is mass-conserving, $\sum_{i=1}^n x_i$ is constant on $P_A(u/|u|_1)$, so this is equivalent to the fact that Shannon entropy $H(x)= -\sum_{i=1}^n x_i\log x_i$ is  increasing, and attains its unique maximum value in $P_A(u/|u|_1)$ at $x(\infty)$. By Theorem~\ref{thm:birch}, the point $x(\infty)$ must be the maximum likelihood distribution $\hat{p}(u)$.
\end{proof}

As the reversible reactions in $\mathcal{S}_{MLE}$ approach closer and closer to equilibrium, we wish to absorb the values of the $X$ variables into reaction rates and pretend that the irreversible reactions are reactions only in the $\theta$ variables. This has the advantage that we can treat this pretend reaction system in the $\theta$ variables as a perturbation of a reversible, detailed balanced system. We can then hope to employ Theorem~\ref{thm:pgac} and conclude global convergence for these irreversible reactions, and hence for $\mathcal{S}_{MLE}$. 

One small technical point deserves mention. The pretend reaction system in the $\theta$ variables is not a reaction system since the rates are not real numbers but functions of time. This will not trouble us. We have already provisioned for this in Definition~\ref{def:pert} by allowing perturbations of reaction systems to be differential inclusions.

\begin{theorem}[The reaction system $\mathcal{S}_{MLE}(A,B,B')$ computes the Maximum Likelihood Estimator]\label{thm:MLE}
Fix a design matrix  $A= (a_{ij})_{m\times n}$, a basis $B$ for the free group $\mathbb{Z}^n\cap\ker A$, and a nonzero data vector $u\in\mathbb{Z}^n_{\geq 0}$. Let $x(t) = (x_1(t), x_2(t),\dots, x_n(t),\theta_1(t),\theta_2(t),\dots, \theta_m(t))$ be a solution to the mass-action differential equations for the reaction system $\mathcal{S}_{MLE}(A,B,B')$ with initial conditions $x(0) = u/|u|_1$ and $\theta(0)=0$. Then $x(\infty):=\lim_{t\to\infty} x(t)$ exists and equals the maximum likelihood distribution $\hat{p}(u)$, and $\theta(\infty):=\lim_{t\to\infty} \theta(t)$ exists and equals the maximum likelihood estimator $\hat{\theta}(u)$.
\end{theorem}
\begin{proof}[Proof sketch]
Fix $u$ and let $\hat{p}=\hat{p}(u)$ and $\hat{\theta}=\hat{\theta}(u)$. Note that for the species $X_1, X_2,\dots, X_n$, the differential equations for $\mathcal{S}_{MLE}(A,B)$ and $\mathcal{S}_{MLD}(A,B,B')$ are identical, since these species appear purely catalytically in the reactions that belong to $\mathcal{R}_{MLE}(A,B,B')\setminus \mathcal{R}_{MLD}(A,B)$. Hence $x(\infty)=\hat{p}(u)$ follows from Theorem~\ref{thm:MLD}. 

To see that $\theta(\infty)=\hat{\theta}$, let us first allow the $X$ species to reach equilibrium, then treat the $\theta$ system with replacing the $X$ species by rate constants representing their values at equilibrium. The system $\Theta_{MLE}(A,B,B',x(\infty))$ obtained in this way in only the $\theta$ species is a reaction system with the reactions
\begin{itemize}\label{sys:theta}
\item For each column $j\in B'$ of $A$, a reaction $\sum_{i=1}^m a_{ij} \theta_i \to 0$ of rate $1$
\item For each column $j\in B'$ of $A$, a reaction $0 \to  \sum_{i=1}^m a_{ij} \theta_i$ of rate $x_j(\infty)$.
\end{itemize}
This is a reversible reaction system, and the maximum likelihood estimators $\hat{\theta}$ are precisely the points of detailed balance for this system, where we are using the fact that $B'$ was a maximal linearly-independent set of the columns of $A$. In addition, this system has no siphons since if species $\theta_i$ is absent, and $a_{ij}>0$ then $\theta_i$ will immediately be produced by the reaction $0 \to  \sum_{i'=1}^m a_{i'j} \theta_{i'}$. (We are assuming $A$ has no $0$ row. If $A$ has a $0$ row, we can ignore it anyway.) It follows from Theorem~\ref{thm:gac} that this system is globally asymptotically stable, and every trajectory approaches a maximum likelihood estimator $\hat{\theta}$.

Our actual system may be viewed as a perturbation of the system $\Theta_{MLE}(A,B,B',x(\infty))$. Consider any trajectory $(x(t), \theta(t))$ to $\mathcal{S}_{MLE}(A,B,B')$ starting at $(u/|u|_1,0)$. We are going to consider the projected trajectory $(\mathbb{R}_{\geq},\theta)$. We now show that it is possible to choose appropriate $t_i$ and $\delta_i$ so that $((t_i,\infty),\theta(t))$ is a trajectory of a $\delta_i$-perturbation of $\Theta_{MLE}(A,B,B', x(\infty))$, for $i=1,2,\dots$. 

Wait for a sufficiently large time $t_1$ till $x(t)$ is in a sufficiently small $\delta_1$ neighborhood of $x(\infty)$ which it will never leave. After this time, we obtain a differential inclusion in the $\theta$ species with the mass-action equations~\ref{eqn:ma} for the reactions
\begin{itemize}\label{sys:xtheta}
\item For each column $j$ of $A$, a reaction $\sum_{i=1}^m a_{ij} \theta_i \to 0$ of rate $1$
\item For each column $j$ of $A$, a reaction $0 \to  \sum_{i=1}^m a_{ij} \theta_i$ with time-varying rate lying in the interval ${(x_j(\infty)-\delta_1, x_j(\infty)+\delta_1)}$.
\end{itemize}
Continuing in this way, we choose a decreasing sequence $\delta_1>\delta_2>\dots>0$ with $\lim_{i\to\infty}\delta_i\to 0$, and corresponding times $t_1<t_2<t_3\dots$ with $\lim_{i\to\infty} t_i\to \infty$ such that after time $t_i$, $x(t)$ is in a $\delta_i$ neighborhood of $x(\infty)$ which it will never leave. Then $((t_i,\infty),\theta(t))$ is a trajectory of the $\delta_i$-perturbation of $\Theta_{MLE}(A,B,B',x(\infty))$. Hence $\theta(t)$ satisfies the conditions of Theorem~\ref{thm:pgac}. Hence $\lim_{t\to\infty}\theta(t)=\hat{\theta}$.
\end{proof} 

\section{Related Work and Conclusions}
The mathematical similarities of both log-linear statistics and reaction networks to toric geometry have been pointed out before~\cite{TDS,miller2011theory}. Craciun et al.~\cite{TDS} refer to the steady states of complex-balanced reaction networks as {\it Birch points} ``to highlight the parallels'' with algebraic statistics. This paper develops on these observations, and serves to flesh out this mathematical parallel into a scheme for molecular computation.

Various building blocks for molecular computation that assume mass-action kinetics have been proposed before. We briefly review some of these proposals.

In \cite{napp2013message}, Napp and Adams model molecular computation with mass-action kinetics, as we do here. They propose a molecular scheme to implement message passing schemes in probabilistic graphical models. The goal of their scheme is to convert a factor graph into a reaction network that encodes the single-variable marginals of the joint distribution as steady state concentrations. In comparison, the goal of our scheme is to do statistical inference and compute maximum likelihood estimators for log-linear models. Napp and Adams focus on the ``forward model'' task of how a given data-generating process (a factor graph) can lead to observed data, whereas our focus is on the ``backward model'' task of inference, going from the observed data to the data-generating process. Further our scheme couples the deep role that MaxEnt algorithms play in Machine Learning with MaxEnt's roots in the Second Law of Thermodynamics whereas Napp and Adams are drawing their inspiration from variable elimination implemented via message passing which has its roots in Boolean constraint satisfaction problems. 


Qian and Winfree~\cite{seesawgates,qian2011scaling} have proposed a DNA gate motif that can be composed to build large circuits, and have experimentally demonstrated molecular computation of a Boolean circuit with around 30 gates. In comparison, our scheme natively employs a continuous-time dynamical system to do the computation, without a Boolean abstraction.

Taking a control theory point of view, Oishi and Klavins~\cite{Klavins_2011Biomolecular} have proposed a scheme for implementing linear input/output systems with reaction networks. Note that for a given matrix $A$, the set of maximum likelihood distributions is usually not linear, but log-linear.

Daniel et al.\cite{daniel2013synthetic} have demonstrated an in vivo implementation of feedback loops, exploiting analogies with electronic circuits. It is possible that the success of their schemes is also related to the toric nature of mass-action kinetics.

Buisman et al.~\cite{buisman2009computing} have proposed a reaction network scheme for computation of algebraic functions. The part of our scheme which reads out the maximum likelihood estimator from the maximum likelihood distribution bears some similarity to their work.

One limitation of our present work is that the number of columns of the matrix $A$ can become very large, for example $2^{|V|}$ for a graphical model with $V$ nodes. Since the number of species and number of reactions both depend on the number of columns of $A$, this can require an exponentially large reaction network which may become impractical. One direction for future work is to extend our scheme by specifying a reaction network that computes maximum likelihood for graphical models.

We have some freedom in our scheme in the choice of basis sets $B$ and $B'$. In any chemical implementation of this work, there might be opportunity for optimization in choice of basis.
\paragraph{Acknowledgements:} I thank Nick S. Jones, Anne Shiu, Abhishek Behera, Ezra Miller, Thomas Ouldridge, Gheorghe Craciun, and Bence Melykuti for useful discussions.
  
\bibliographystyle{plain}
\bibliography{../eventsystems}
\end{document}